\documentclass[a4,a4wide]{article}
\usepackage{aaai}
\usepackage{times}
\usepackage{helvet}
\usepackage{courier}

\usepackage{amssymb, amsmath, latexsym}
\usepackage{graphicx}
\usepackage{url, qtree}
\usepackage{centernot}

\usepackage{mathrsfs}

\renewcommand{\emptyset}{\varnothing}
\newtheorem{definition}{Definition}
\newtheorem{theorem}{Theorem}
\newtheorem{proposition}{Proposition}
\newtheorem{example}{Example}
\newtheorem{lemma}{Lemma}

\newenvironment{proof}[1][]{\begin{trivlist} \parindent = 0pt
\parskip = 1.5ex \item[\hskip \labelsep{\bf Proof #1}]}
{\hspace*{\fill} $\Box$ \end{trivlist}}

\title{An argumentation system for reasoning with \\
conflict-minimal paraconsistent $\mathcal{ALC}$}

\author{Wenzhao Qiao \and Nico Roos\\
Department of Knowledge Engineering, Maastricht University\\
Bouillonstraat 8-10, 6211 LH Maastricht, The Netherlands\\
\{wenzhao.qiao,roos\}@maastrichtuniversity.nl}

\begin{document}
\nocopyright
\maketitle
\begin{abstract}
The semantic web is an open and distributed environment in which
it is hard to guarantee consistency of knowledge and information.
Under the standard two-valued semantics everything is entailed if
knowledge and information is inconsistent. The semantics of the
paraconsistent logic LP offers a solution. However, if the
available knowledge and information is consistent, the set of
conclusions entailed under the three-valued semantics of the
paraconsistent logic LP is smaller than the set of conclusions
entailed under the two-valued semantics. Preferring
conflict-minimal three-valued interpretations eliminates this
difference.

Preferring conflict-minimal interpretations introduces
non-monotonicity. To handle the non-monotonicity, this paper
proposes an assumption-based argumentation system. Assumptions
needed to close branches of a semantic tableaux form the
arguments. Stable extensions of the set of derived arguments
correspond to conflict minimal interpretations and conclusions
entailed by all conflict-minimal interpretations are supported by
arguments in all stable extensions.
\end{abstract}

\section{Introduction} \label{sec:introduction}
In the semantic web, the description logics $\mathcal{SHOIN}(D)$
and $\mathcal{SROIQ}(D)$ are the standard for describing
ontologies using the TBox, and information using the Abox. Since
the semantic web is an open and distributed environment, knowledge
and information originating from different sources need not be
consistent. In case of inconsistencies, no useful conclusion can
be derived when using a standard two-valued semantics. Everything
is entailed because the set of two-valued interpretations is
empty. Resolving the inconsistencies is often not an option in an
open and distributed environment. Therefore, methods that allow us
to derive useful conclusions in the presence of inconsistencies
are preferred.

One possibility to draw useful conclusions from inconsistent
knowledge and information is by focussing on conclusions supported
by all maximally consistent subsets. This approach was first
proposed by Rescher \shortcite{Res-64} and was subsequence worked
out further by others \cite{Bre-89,Roo-88a,Roo-92}. A simple
implementation of this approach focusses on conclusions entailed
by the intersection of all maximally consistent subsets. Instead
of focussing on the intersection of all maximally consistent
subsets, one may also consider a single consistent subset for each
conclusion \cite{Poo-88,HHT-05}. For conclusions entailed by all
(preferred) maximally consistent subsets of the knowledge and
information, a more sophisticated approach is needed. An
argumentation system for this more general case has been described
by Roos \shortcite{
Roo-92}. Since these approaches need to identify consistent
subsets of knowledge and information, they are
\emph{non-monotonic}.

A second possibility for handling inconsistent knowledge and
information is by replacing the standard two-valued semantics by a
three-valued semantics such as the semantics of the paraconsistent
logic LP \cite{Pri-89}. An important advantage of this
paraconsistent logic over the maximally consistent subset approach
is that the entailment relation is \emph{monotonic}. A
disadvantage is that consistent knowledge and information entail
less conclusions when using the three-valued semantics than when
using the two-valued semantics. Conflict-minimal interpretations
reduce the gap between the sets of conclusions entailed by the two
semantics \cite{Pri-89,Pri-91}. Priest \shortcite{Pri-91} calls
resulting logic: LPm. The conflict-minimal interpretations also
makes LPm \emph{non-monotonic} \cite{Pri-91}.

In this paper we present an argumentation system for conclusions
entailed by conflict-minimal interpretations of the description
logic $\mathcal{ALC}$ \cite{SS-91} when using the semantics of the
paraconsistent logic LP. We focus on $\mathcal{ALC}$ instead of
the more expressive logics $\mathcal{SHOIN}(D)$ and
$\mathcal{SROIQ}(D)$ to keep the explanation simple. The described
approach can also be applied to more expressive description
logics.

The proposed approach starts from a semantic tableaux method for
the paraconsistent logic LP described by Bloesch
\shortcite{Blo-93}, which has been adapted to $\mathcal{ALC}$. The
semantic tableaux is used for deriving the entailed conclusions
when using the LP-semantics. If a tableaux cannot be closed, the
desired conclusion may still hold in all conflict-minimal
interpretations. The open tableaux enables us to identify
assumptions about conflict-minimality. These assumptions are used
to construct an \emph{assumption-based argumentation system},
which supports conclusions entailed by all conflict minimal
interpretations.

The remainder of the paper is organized as follows. First, we
describe $\mathcal{ALC}$, a three-valued semantics for
$\mathcal{ALC}$ based on the semantics of the paraconsistent logic
LP, and a corresponding semantic tableaux method. Second, we
describe how a semantic tableaux can be used to determine
arguments for conclusions supported by conflict-minimal
interpretations. Subsequently, we present the correctness and
completeness proof of the described approach. Next we describe
some related work. The last section summarizes the results and
points out directions of future work.

\section{Paraconsistent $\mathcal{ALC}$} \label{sec:paracon-logic}
\paragraph{The language of $\mathcal{ALC}$}
We first give the standard definitions of the language of
$\mathcal{ALC}$. We start with defining the set of concepts
$\mathcal{C}$ given the atomic concepts $\mathbf{C}$, the role
relations $\mathbf{R}$, the operators for constructing new
concepts $\neg$, $\sqcap$ and $\sqcup$, and the quantifiers
$\exists$ and $\forall$. Moreover, we introduce to special
concepts, $\top$ and $\bot$, which denote \emph{everything} and
\emph{nothing}, respectively.
\begin{definition}\label{def:standard-alc-concepts}
Let $\mathbf{C}$ be a set of atomic concepts and let $\mathbf{R}$
be a set of atomic roles.

The set of concepts $\mathcal{C}$ is recursively defined as
follows:
\begin{itemize}
\item
$\mathbf{C} \subseteq \mathcal{C}$; i.e. atomic concepts are
concepts.
\item
$\top \in \mathcal{C}$ and $\bot \in \mathcal{C}$.
\item
If $C \in \mathcal{C}$ and $D \in \mathcal{C}$, then $\neg C \in
\mathcal{C}$, $C \sqcap D \in \mathcal{C}$ and $C \sqcup D \in
\mathcal{C}$.
\item
If $C \in \mathcal{C}$ and $R \in \mathbf{R}$, then $\exists R.C
\in \mathcal{C}$ and $\forall R.C \in \mathcal{C}$.
\item
Nothing else belongs to $\mathcal{C}$.
\end{itemize}
\end{definition}

In the description logic $\mathcal{ALC}$, we have two operators:
$\sqsubseteq$ and $=$, for describing a relation between two
concepts:
\begin{definition}\label{def:standard-alc-relations}
If $\{ C, D \} \subseteq \mathcal{C}$, then we can formulate the
following relations (terminological definitions):
\begin{itemize}
\item
$C \sqsubseteq D$; i.e., $C$ is subsumed by $D$,
\item
$C = D$; i.e., $C$ is equal to $D$.
\end{itemize}
\end{definition}
A finite set $\mathcal{T}$ of terminological definitions is called
a \emph{TBox}.

In the description logic $\mathcal{ALC}$, we also have an operator
``$:$", for describing that an individual from the set of
individual names $\mathbf{N}$ is an instance of a concept, and
that a pair of individuals is an instance of a role.

\begin{definition}\label{def:standard-alc-assertions}
Let $\{ a, b \} \subseteq \mathbf{N}$ be two individuals, let $C
\in \mathcal{C}$ be a concept and let $R \in \mathbf{R}$ be a
role. Then assertions are defined as:
\begin{itemize}
\item
$a:C$
\item
$(a, b):R$
\end{itemize}
\end{definition}
A finite set $\mathcal{A}$ of assertions is called an \emph{ABox}.

A \emph{knowledge base} $\mathcal{K} = (\mathcal{T}, \mathcal{A})$
is a tuple consisting of a TBox $\mathcal{T}$ and an ABox
$\mathcal{A}$. In this paper we will denote elements of the TBox
and ABox $\mathcal{T} \cup \mathcal{A}$ as propositions.

We define a three-valued semantics for $\mathcal{ALC}$ which is
based on the semantics of the paraconsistent logic $LP$. We do not
use the notation $I = (\Delta,\cdot^I)$ that is often used for the
semantics of description logics. Instead we will use a notation
that is often used for predicate logic because it is more
convenient to describe projections and truth-values.
\begin{definition}\label{def:4-valued-interpretation}
A three-valued interpretation $I = \langle O, \pi \rangle$ is a
couple where $O$  is a non-empty set of objects and $\pi$ is an
interpretation function such that:
\begin{itemize}
\item
for each atomic concept $C \in \mathbf{C}$, $\pi(C) = \langle P, N
\rangle$ where $P, N \subseteq O$ are the positive and negative
instances of the concept $C$, respectively, and where $P\cup N=O$,
\item
for each individual $i \in \mathbf{N}$ it holds that $\pi(i) \in
O$, and
\item
for each atomic role $R \in \mathbf{R}$ it holds that $\pi(R)
\subseteq O \times O$.
\end{itemize}
We will use the projections $\pi(C)^+=P$ and $\pi(C)^-=N$ to
denote the positive and negative instances of a concept $C$,
respectively.
\end{definition}

We do not consider inconsistencies in roles since we cannot
formulate inconsistent roles in $\mathcal{ALC}$. In a more
expressive logic, such as $\mathcal{SROIQ}$, roles may become
inconsistent, for instance because we can specify disjoint roles.

Using the three-valued interpretations $I = \langle O, \pi
\rangle$, we define the interpretations of concepts in
$\mathcal{C}$.
\begin{definition}\label{def:4-valued-extended-interpretation}
The interpretation of a concept $C \in \mathcal{C}$ is defined by
the extended interpretation function $\pi^*$.
\begin{itemize}
\item
$\pi^*(C) = \pi(C)$ \textit{iff} $C \in \mathbf{C}$
\item
$\pi^*(\top) = \langle O, X \rangle$, where $X \subseteq O$
\item
$\pi^*(\bot) = \langle X, O \rangle$, where $X \subseteq O$
\item
$\pi^*(\neg C) = \langle \pi^*(C)^-,\pi^*(C)^+\rangle$
\item
$\pi^*(C \sqcap D) = \langle \pi^*(C)^+ \cap \pi^*(D)^+,
\pi^*(C)^- \cup \pi^*(D)^-\rangle$
\item
$\pi^*(C \sqcup D) = \langle \pi^*(C)^+ \cup \pi^*(D)^+,
\pi^*(C)^- \cap \pi^*(D)^-\rangle$
\item
$\pi^*(\exists R.C) =$ \\
$\langle \begin{array}[t]{l} \{ x \in O \mid
\exists y \in O, (x, y) \in \pi(R)~\text{and}~y \in \pi(C)^+ \},\\
\{ x \in O \mid \forall y \in O, (x, y) \in
\pi(R)~\text{implies}~y \in \pi(C)^- \}~\rangle\end{array}$
\item
$\pi^*(\forall R.C) =$ \\$ \langle \begin{array}[t]{l} \{ x \in O
\mid \forall y \in O, (x, y) \in \pi(R)~\text{implies}~y \in \pi(C)^+ \},\\
\{ x \in O \mid \exists y \in O, (x, y) \in \pi(R)~\text{and}~y
\in \pi(C)^- \}~\rangle \end{array}$
\end{itemize}
\end{definition}
Note that we allow inconsistencies in the concepts $\top$ and
$\bot$. There may not exist a tree-valued interpretation for a
knowledge-base $\mathcal{K} = (\mathcal{T}, \mathcal{A})$ if we
require that $X=\emptyset$. Consider for instance: $a:C$, $a:D$
and $C \sqcap D \sqsubseteq \bot$.

We also use the extended interpretation function $\pi^*$ to define
the truth values of the propositions: $C \sqsubseteq D$, $a:C$ and
$(a, b) : R$. The truth values of the three-valued semantics are
defined using sets of the ``classical" truth values: $t$ and $f$.
We use three sets in the LP-semantics: $\{t\}$, $\{f\}$ and
$\{t,f\}$, which correspond to \textsc{true}, \textsc{false} and
\textsc{conflict}.

\begin{definition}\label{def:4-valued-t-f}
Let $\{ a, b \} \subseteq \mathbf{N}$ be two individuals, let $C
\in \mathcal{C}$ be a concept and let $R \in \mathbf{R}$ be a
role. Then an interpretation $I = \langle O, \pi \rangle$ of
propositions is defined as:
\begin{itemize}
\item
$t \in \pi^*(a:C)$ \textit{iff} $\pi^*(a) \in \pi^*(C)^+$
\item
$f \in \pi^*(a:C)$ \textit{iff} $\pi^*(a) \in \pi^*(C)^-$
\item
$t \in \pi^*(C \sqsubseteq D)$ \textit{iff} 
$\pi^*(C)^+ \subseteq \pi^*(D)^+$,  and \\ 
\hspace*{26mm} $\pi^*(D)^- \subseteq \pi^*(C)^-$
\item
$f \in \pi^*(C \sqsubseteq D)$ \textit{iff} $t \not\in \pi^*(C
\sqsubseteq D)$
\item
$t \in \pi^*(C = D)$ \textit{iff} $\pi^*(C)^+ = \pi^*(D)^+$ and \\
\hspace*{26mm} $\pi^*(D)^- = \pi^*(C)^-$
\item
$f \in \pi^*(C = D)$ \textit{iff} $t \not\in \pi^*(C = D)$
\item
$t \in \pi^*((a, b) : R)$ \textit{iff} $(\pi^*(a), \pi^*(b)) \in
\pi(R)$
\item
$f \in \pi^*((a, b) : R)$ \textit{iff} $(\pi^*(a), \pi^*(b))
\not\in \pi(R)$
\end{itemize}
\end{definition}

The interpretation of the subsumption relation given above was
proposed by Patel-Schneider \shortcite{Pat-89} for their
four-valued semantics. Patel-Schneider's interpretation of the
subsumption relation does not correspond to the material
implication $\forall x [C(x) \to D(x)]$ in first-order logic. The
latter is equivalent to $\forall x [\neg C(x) \vee D(x)]$ under
the two-valued semantics, which corresponds to: ``for every $o\in
O$, $o \in \pi^*(C)^-$ or $o \in \pi^*(D)^+$'' under the
three-valued semantics. No conclusion can be drawn from $a:C$ and
$C \sqsubseteq D$ under the three-valued semantics since there
always exists an interpretation such that $\pi(a:C)=\{t,f\}$.

The entailment relation can be defined using the interpretations
of propositions.
\begin{definition}\label{def:4-valued-t-f}
Let $I = \langle O, \pi \rangle$ be an interpretation, let
$\varphi$ be a proposition, and let $\Sigma$ be a set of
propositions. The the entailment relation is defined as:
\begin{itemize}
\item
$I \models \varphi$ iff $t \in \pi^*(\varphi)$
\item
$I \models \Sigma$ iff $t \in \pi^*(\sigma)$ for every $\sigma\in
\Sigma$.
\item
$\Sigma \models \varphi$ iff $I \models \Sigma$ implies $I \models
\varphi$ for each interpretation $I$
\end{itemize}
\end{definition}

\paragraph{Semantic tableaux}
We use a semantic tableaux method that is based on the semantic
tableaux method for LP described by Bloesch \shortcite{Blo-93}.
This tableaux method will enable us to identify the assumptions
underlying relevant conflict minimal interpretations.

Bloesch proposes to label every proposition in the tableaux with
either the labels $\mathbb{T}$ (at least true), $\mathbb{F}$ (at
least false), or their complements $\overline{\mathbb{T}}$ and
$\overline{\mathbb{F}}$, respectively. So, $\mathbb{T} \varphi$
corresponds to $t \in \pi(\varphi)$, $\overline{\mathbb{T}}
\varphi$ corresponds to $t \not\in \pi(\varphi)$, $\mathbb{F}
\varphi$ corresponds to $f \in \pi(\varphi)$, and
$\overline{\mathbb{F}} \varphi$ corresponds to $f \not\in
\pi(\varphi)$.

Although we do not need it in the semantic tableaux, we also make
use of $\mathbb{C} \varphi$ and $\overline{\mathbb{C}} \varphi$,
which corresponds semantically with $\{t,f\} = \pi(\varphi)$ and
$\{t,f\} \not= \pi(\varphi)$, respectively. So, $\mathbb{C}
\varphi$ is equivalent to: `$\mathbb{T} \varphi$ \textit{and}
$\mathbb{F} \varphi$', and $\overline{\mathbb{C}} \varphi$ is
equivalent to: `$\overline{\mathbb{T}} \varphi$ \textit{or}
$\overline{\mathbb{F}} \varphi$'.

To prove that $\Sigma \models \varphi$ using Bloesch's tableaux
method \cite{Blo-93}, we have to show that a tableaux with root
$\Gamma = \{ \mathbb{T} \sigma \mid \sigma \in \Sigma \} \cup
\mathbb{\overline{T}} \varphi$ closes. The tableaux closes if
every branch has a node in which for some proposition $\alpha$ the
node contains: ``$\mathbb{T} \alpha$ \textit{and}
$\overline{\mathbb{T}} \alpha$'', or ``$\mathbb{F} \alpha$
\textit{and} $\overline{\mathbb{F}} \alpha$'', or
``$\overline{\mathbb{T}} \alpha$ \textit{and}
$\overline{\mathbb{F}} \alpha$''.

Based on Bloesch's semantic tableaux method for LP, the following
tableaux rules have been formulated. The soundness and
completeness of the set of rules are easy to prove.

\begin{center}
\begin{tabular}{cccc}
 $\displaystyle \frac{\mathbb{T} \ a:\neg C}{\mathbb{F} \ a:C} $
 & $\displaystyle \frac{\mathbb{\overline{T}} \ a:\neg C}{\mathbb{\overline{F}} \ a:C} $
 & $\displaystyle \frac{\mathbb{F} \ a:\neg C}{\mathbb{T} \ a:C} $
 & $\displaystyle \frac{\mathbb{\overline{F}} \ a:\neg C}{\mathbb{\overline{F}} \ a:C} $
\end{tabular}
\begin{tabular}{cc}
 $\displaystyle \frac{\mathbb{T} \ a:C \sqcap D}{\mathbb{T} \ a:C, \mathbb{T} \ a:D} $
 & $\displaystyle \frac{\mathbb{\overline{T}} \ a:C \sqcap D}
 {\mathbb{\overline{T}} \ a:C \mid \mathbb{\overline{T}} \ a:D} $
 \\
 $\displaystyle \frac{\mathbb{F} \ a:C \sqcap D}{\mathbb{F} \ a:C \mid \mathbb{F} \ a:D} $
 & $\displaystyle \frac{\mathbb{\overline{F}} \ a:C\sqcap D}
 {\mathbb{\overline{F}} \ a:C, \mathbb{\overline{F}} \ a:D} $
 \\
 $\displaystyle \frac{\mathbb{T} \ a:C \sqcup D}{\mathbb{T} \ a:C \mid \mathbb{T} \ a:D} $
 & $\displaystyle \frac{\mathbb{\overline{T}} \ a:C \sqcup D}
 {\mathbb{\overline{T}} \ a:C, \mathbb{\overline{T}} \ a:D} $
 \\
 $\displaystyle \frac{\mathbb{F} \ a:C \sqcup D}{\mathbb{F} \ a:C, \mathbb{F} \ a:D} $
 & $\displaystyle \frac{\mathbb{\overline{F}} \ a:C\sqcup D}
 {\mathbb{\overline{F}} \ a:C \mid \mathbb{\overline{F}} \ a:D} $
 \\
 $\displaystyle \frac{\mathbb{T} \ a:\exists r.C}{\mathbb{T} \ (a, x):r, \mathbb{T} \ x:C} $
 & $\displaystyle \frac{\mathbb{\overline{T}} \ a:\exists r.C, \mathbb{T} \ (a, b):r}
 {\mathbb{\overline{T}} \ b:C} $
\end{tabular}
\begin{tabular}{cc}
 $\displaystyle \frac{\mathbb{F} \ a:\exists r.C, \mathbb{T} \ (a, b):r}{\mathbb{F} \ b:C} $
 & $\displaystyle \frac{\mathbb{\overline{F}} \ a:\exists r.C}
 {\mathbb{T} \ (a, x):r, \mathbb{\overline{F}} \ x:C} $
 \\
 $\displaystyle \frac{\mathbb{T} \ a:\forall r.C, \mathbb{T} \ (a, b):r}{\mathbb{T} \ b:C} $
 & $\displaystyle \frac{\mathbb{\overline{T}} \ a:\forall r.C}
 {\mathbb{T} \ (a, x):r, \mathbb{\overline{T}} \ x:C} $
 \\
 $\displaystyle \frac{\mathbb{F} \ a:\forall r.C}{\mathbb{T} \ (a, b):r, \mathbb{F} \ b:C} $
 & $\displaystyle \frac{\mathbb{\overline{F}} \ a:\forall r.C, \mathbb{T} \ (a, b):r}
 {\mathbb{\overline{F}} \ b:C} $
\end{tabular}
\end{center}
The individual $a$ in the following tableaux rules for the
subsumption relation must be an existing individual name, while
the individual $x$ must be a new individual name.
\begin{center}
\begin{tabular}{cc}
$\displaystyle \frac{\mathbb{T} \ C \sqsubseteq
D}{\mathbb{\overline{T}} \ a:C \mid \mathbb{T} \ a:D} $
 & $\displaystyle \frac{\mathbb{T} \ C \sqsubseteq D}
 {\mathbb{\overline{F}} \ a:D \mid \mathbb{F} \ a:C} $ \\
\end{tabular}
\begin{tabular}{cc}
$\displaystyle \frac{\mathbb{\overline{T}} \ C \sqsubseteq
D}{\mathbb{T} \ x:C, \mathbb{\overline{T}} \ x:D \mid \mathbb{F} \
x:D, \mathbb{\overline{F}} \ x:C} $ \\
\end{tabular}
\begin{tabular}{cc}
$\displaystyle \frac{\mathbb{T} \ C = D}{\mathbb{T} \ C
\sqsubseteq D, \mathbb{T} \ D \sqsubseteq C} $
 & $\displaystyle \frac{\overline{\mathbb{T}} \ C = D}
 {\mathbb{\overline{T}} \ C \sqsubseteq D \mid
 \mathbb{\overline{T}} \ D \sqsubseteq C} $ \\
\end{tabular}
\end{center}

An important issue is guaranteeing that the constructed semantic
tableaux is always finite. The \emph{blocking} method described by
\cite{BDS-93,BBH-96} is used to guarantee the construction of a
finite tableaux. A rule that is \emph{blocked},  may not be not be
used in the construction of the tableaux.

\begin{definition}
\label{def:blocking-methods} Let $\Gamma$ be a node of the
tableau, and let $x$ and $y$ be two individual names. Moreover,
let $\Gamma(x) = \{ \mathbb{L} x : C \mid \mathbb{L} x : C \in
\Gamma \}$.
\begin{itemize}
\item
$x <_r y$ if $(x, y):R \in \Gamma$ for some $R \in \mathbf{R}$.
\item
$y$ is blocked if there is an individual name $x$ such that: $x
<_r^+ y$ and $\Gamma(y) \subseteq \Gamma(x)$, or $x <_r y$ and $x$
is blocked.
\end{itemize}
\end{definition}

\paragraph{Conflict Minimal Interpretations}
A price that we pay for changing to the three-valued LP-semantics
in order to handle inconsistencies is a reduction in the set of
entailed conclusions, even if the knowledge and information is
consistent.
\begin{example} \label{ex:1}
The set of propositions $\Sigma=\{a:\neg C,a:C \sqcup D\}$ does
not entail $a:D$ because there exists an interpretation $I=\langle
O,\pi \rangle$ for $\Sigma$ such that $\pi(a:C)=\{t,f\}$ and
$\pi(a:D)=\{f\}$.
\end{example}
Priest \shortcite{Pri-89,Pri-91} points out that more  useful
conclusions can be derived from the paraconsistent logic LP if we
would prefer conflict-minimal interpretations. The resulting logic
is LPm. Here we follow the same approach. First, we define a
conflict ordering on interpretations.

\begin{definition}\label{def:conflict-ordering}
Let $\mathbf{C}$ be a set of atomic concepts, let $\mathbf{N}$ be
a set of individual names, and let $I_1$ and $I_2$ be two
three-valued interpretations.

The interpretation $I_1$ contains less conflicts than the
interpretation $I_2$, denoted by $I_1 <_c I_2 $, iff:
\begin{center}
$\{ a:C \mid a \in \mathbf{N}, C \in \mathbf{C}, \pi_1(a:C) = \{t,f\} \} \subset$ \\
\hspace*{10mm} $\{ a:C \mid a \in \mathbf{N}, C \in \mathbf{C},
\pi_2(a:C) = \{t,f\} \}$
\end{center}
\end{definition}
The following example gives an illustration of a conflict ordering
for the set of propositions of Example \ref{ex:1}.
\begin{example}\label{ex:conflict-ordering}
Let $\Sigma=\{a:\neg C,a:C \sqcup D\}$ be a set of propositions
and let $I_1$, $I_2$, $I_3$, $I_4$ and $I_5$ be five
interpretations such that:
\begin{itemize}
\item
$\pi^*_1(a:C)=\{f\}$, $\pi^*_1(a:D)=\{t\}$,
\item
$\pi^*_2(a:C)=\{f\}$, $\pi^*_2(a:D)=\{t,f\}$.
\item
$\pi^*_3(a:C)=\{t,f\}$, $\pi^*_3(a:D)=\{t\}$,
\item
$\pi^*_4(a:C)=\{t,f\}$, $\pi^*_4(a:D)=\{f\}$,
\item
$\pi^*_5(a:C)=\{t,f\}$, $\pi^*_5(a:D)=\{t,f\}$.
\end{itemize}

Then $I_1 <_c I_2$, $I_1 <_c I_3$, $I_1 <_c I_4$, $I_1 <_c I_5$,
$I_2 <_c I_5$, $I_3 <_c I_5$  and $I_4 <_c I_5$.
\end{example}

Using the conflict ordering, we define the conflict minimal
interpretations.
\begin{definition}\label{def:conflict-minimal-interpretation}
Let $I_1$ be a three-valued interpretation and let $\Sigma$ be a
set of propositions.

$I_1$ is a conflict minimal interpretation of $\Sigma$, denoted by
$I_1 \models_{<_c} \Sigma$, iff $I_1 \models \Sigma$ and for no
interpretation $I_2$ such that $I_2 <_c I_1$, $I_2 \models \Sigma$
holds.
\end{definition}
In Example \ref{ex:conflict-ordering}, $I_1$ is the only
conflict-minimal interpretation.

The conflict-minimal entailment of a proposition by a set of
propositions can now be defined.
\begin{definition}\label{def:conflict-minimal-entailment}
Let $\Sigma = (\mathcal{T} \cup \mathcal{A})$ be a set of
propositions and let $\varphi$ be a proposition.

$\Sigma$ entails conflict-minimally the proposition $\varphi$,
denoted by $\Sigma \models_{<_c} \varphi$, iff for every
interpretation $I$, if $I \models_{<_c} \Sigma$, then $I \models
\varphi$.
\end{definition}
The conflict-minimal interpretations in Example
\ref{ex:conflict-ordering} entail the conclusion $a:D$.

\paragraph{The subsumption relation}
The conflict-minimal interpretations enables us to use an
interpretation of the subsumption relation based on the material
implication.
\begin{itemize}
\item
For every $o\in O$, $o \in \pi^*(C)^-$ or $o \in \pi^*(D)^+$
\end{itemize}
This semantics of the subsumption relation resolves a problem with
the semantics of Patel-Schneider \shortcite{Pat-89}. Under
Patel-Schneider's semantics, $\{ a:C,a:\neg C, C \sqsubseteq D\}$
entails $a:D$. This entailment is undesirable if information about
$a:C$ is contradictory.

The tableaux rules of the new interpretation are:
\begin{center}
\begin{tabular}{cc}
 $\displaystyle \frac{\mathbb{T} \ C \sqsubseteq D}
 {\mathbb{F} \ a:C \mid \mathbb{T} \ a:D}$
 & $\displaystyle \frac{\mathbb{\overline{T}} \ C \sqsubseteq D}
 {\mathbb{T} \ a:C, \mathbb{F} \ a:D}$
\end{tabular}
\end{center}

\section{Arguments for conclusions supported
by conflict minimal interpretations}\label{sec:reasoning}

The conflict-minimal interpretations of a knowledge base entail
more useful conclusions.  Unfortunately, focusing on conclusions
supported by conflict-minimal interpretations makes the reasoning
process \emph{non-monotonic}. Adding the assertion $a:\neg D$ to
the set of propositions in Example \ref{ex:conflict-ordering}
eliminates interpretations $I_1$ and $I_3$, which includes the
only conflict-minimal interpretation $I_1$. The interpretations
$I_2$ and $I_4$ are the new conflict-minimal interpretations.
Unlike the original conflict-minimal interpretation $I_1$, the new
conflict-minimal interpretations $I_2$ and $I_4$ do not entail
$a:D$.

Deriving conclusions supported by the conflict-minimal
interpretations is problematic because of the non-monotonicity.
The modern way to deal with non-monotonicity is by giving an
argument supporting a conclusion, and subsequently verifying
whether there are no counter-arguments \cite{Dun-95}. Here we will
follow this argumentation-based approach.

We propose an approach for deriving arguments that uses the
semantic tableaux method for our paraconsistent logic as a
starting point. The approach is based on the observation that an
interpretation satisfying the root of a semantic tableaux will
also satisfy one of the leafs. Now suppose that the only leafs of
a tableaux that are not closed; i.e., leaf in which we do not have
``$\mathbb{T} \alpha$ \textit{and} $\overline{\mathbb{T}}
\alpha$'' or ``$\mathbb{F} \alpha$ \textit{and}
$\overline{\mathbb{F}} \alpha$'' or ``$\overline{\mathbb{T}}
\alpha$ \textit{and} $\overline{\mathbb{F}} \alpha$'', are leafs
in which ``$\mathbb{T} \alpha$ \textit{and} $\mathbb{F} \alpha$''
holds for some proposition $\alpha$. So, in every open branch of
the tableaux, $\mathbb{C} \alpha$ holds for some proposition
$\alpha$. If we can assume that there are no conflicts w.r.t.~each
proposition $\alpha$ in the conflict-minimal interpretations, then
we can also close the open branches. The set of assumptions
$\overline{\mathbb{C}} \alpha$, equivalent to
``$\overline{\mathbb{T}} \alpha$ \textit{or}
$\overline{\mathbb{F}} \alpha$'', that we need to close the open
branches, will be used as the argument for the conclusion
supported by the semantic tableaux.

An advantage of the proposed approach is that there is no need to
consider arguments if a conclusion already holds without
considering conflict-minimal interpretations.

A branch that can be closed \emph{assuming} that the
conflict-minimal interpretations contain \textbf{no} conflicts
with respect to the proposition $\alpha$; i.e., assuming
$\overline{\mathbb{C}} \alpha$, will be called a \emph{weakly
closed} branch. We will call a tableaux \emph{weakly closed} if
some branches are weakly closed and all other branches are closed.
If we can (weakly) close a tableaux for $\Gamma = \{ \mathbb{T}
\sigma \mid \sigma \in (\mathcal{T} \cup \mathcal{A}) \} \cup
\mathbb{\overline{T}} \varphi$, we consider the set of assumptions
$\overline{\mathbb{C}} \alpha$ needed to weakly close the
tableaux, to be the argument supporting $\Sigma \models_{\leq_c}
\varphi$. Example \ref{ex:tableaux1} gives an illustration.

\begin{example} \label{ex:tableaux1}
Let $\Sigma=\{a:\neg C,a:C \sqcup D\}$ be a set of propositions.
To verify whether $a:D$ holds, we may construct the following
tableaux:

\Tree
[.{$\mathbb{T} \ a:\neg C$ \\
$\mathbb{T} \ a:C \sqcup D$ \\
\fbox{$\mathbb{\overline{T}} \ a:D$}}
   [.{$\mathbb{F} \ a:C$}
      [.{$\mathbb{T} \ a:C$}
         [.{$\otimes_{[a:C]}$}
         ]
      ]
      [.{$\mathbb{T} \ a:D$}
         [.{$\times$}
         ]
      ]
   ]
]

\noindent Only the left branch is \emph{weakly closed} in this
tableau. We assume that the assertion $a:C$ will \emph{not} be
assigned \textsc{conflict} in any conflict-minimal interpretation.
That is, we assume that $\overline{\mathbb{C}} \ a:C$ holds.
\end{example}

In the following definition of an argument, we consider arguments
for $\mathbb{T}\varphi$ and $\mathbb{F}\varphi$.
\begin{definition}
Let $\Sigma$ be set of propositions and let $\varphi$ a
proposition. Moreover, let $\mathscr{T}$ be a (weakly) closed
semantic tableaux with root $\Gamma = \{ \mathbb{T} \sigma \mid
\sigma \in \Sigma \} \cup \mathbb{\overline{L}} \varphi$ and
$\mathbb{L} \in \{ \mathbb{T},\mathbb{F} \}$. Finally, let $\{
\overline{\mathbb{C}} \alpha_1, \ldots, \overline{\mathbb{C}}
\alpha_k \}$ be the set of assumptions on which the closures of
weakly closed branches are based.

Then $A=(\{ \overline{\mathbb{C}} \alpha_1, \ldots,
\overline{\mathbb{C}} \alpha_k \}, \mathbb{L}\varphi)$ is an
argument.
\end{definition}

The next step is to verify whether the assumptions:
$\overline{\mathbb{C}} \alpha_i$ are valid. If one of the
assumptions does not hold, we have a \emph{counter-argument} for
our argument supporting $\Sigma \models_{\leq_c} \varphi$. To
verify the correctness of an assumption, we add the assumption to
$\Sigma$. Since an assumption $\overline{\mathbb{C}} \alpha$ is
equivalent to: ``$\overline{\mathbb{T}} \alpha$ \textit{or}
$\overline{\mathbb{F}} \alpha$'', we can consider
$\overline{\mathbb{T}} \alpha$ and $\overline{\mathbb{F}} \alpha$
separately. Example \ref{ex:tableaux2} gives an illustration for
the assumption $\overline{\mathbb{C}} \ a:C$ used in Example
\ref{ex:tableaux1}.

\begin{example} \label{ex:tableaux2}
Let $\Sigma=\{a:\neg C,a:C \sqcup D\}$ be a set of propositions.
To verify whether the assumption $\overline{\mathbb{C}} \ a:C$
holds in every conflict minimal interpretation, we may construct a
tableaux assuming $\overline{\mathbb{T}} \ a:C$ and a tableaux
assuming $\overline{\mathbb{F}} \ a:C$:

\centerline{
\begin{tabular}{ccc}
\Tree
[.{$\mathbb{T} \ a:\neg C$ \\
$\mathbb{T} \ a:C \sqcup D$ \\
\fbox{$\mathbb{\overline{T}} \ a:C$}}
   [.{$\mathbb{F} \ a:C$}
     [.{$\mathbb{T} \ a:C$}
         [.{$\times$}
         ]
      ]
      [.{$\mathbb{T} \ a:D$}
      ]
    ]
] && \Tree
[.{$\mathbb{T} \ a:\neg C$ \\
$\mathbb{T} \ a:C \sqcup D$ \\
\fbox{$\mathbb{\overline{F}} \ a:C$}}
   [.{$\mathbb{F} \ a:C$}
      [.{$\times$}
      ]
   ]
]
\end{tabular} }
\noindent The right branch of the first tableaux cannot be closed.
Therefore, the assumption $\mathbb{\overline{T}} \ a:C$ is valid,
implying that the assumption $\mathbb{\overline{C}} \ a:C$ is also
valid. Hence, there exists \emph{no} counter-argument.
\end{example}

Since the validity of assumptions must be verified with respect to
conflict-minimal interpretations, assumptions may also be used in
the counter-arguments. This implies that we may have to verify
whether there exists a counter-argument for a counter-argument.
Example \ref{ex:tableaux3} gives an illustration.

\begin{example} \label{ex:tableaux3}
Let $\Sigma = \{ a:\neg C, a:C \sqcup D, a:\neg D \sqcup E, a:\neg
E \}$ be a set of propositions. To verify whether $a:D$ holds, we
may construct the following tableaux:

\Tree
[.{$\mathbb{T} \ a:\neg C$ \\
$\mathbb{T} \ a:C \sqcup D$ \\
$\mathbb{T} \ a:\neg D \sqcup E$ \\
$\mathbb{T} \ a:\neg E$ \\
\fbox{$\mathbb{\overline{T}} \ a:D$}}
   [.{$\mathbb{T} \ a:C$}
      [.{$\mathbb{F} \ a:C$}
         [.{$\otimes_{[a:C]}$}
         ]
      ]
   ]
   [.{$\mathbb{T} \ a:D$}
      [.{$\times$}
      ]
   ]
]

This weakly closed tableaux implies the argument $A_0 = (\{
\overline{\mathbb{C}} \ a:C \}, \mathbb{T} \ a:D )$. Next, we have
to verify whether there exists a counter-argument for $A_0$. To
verify the existence of a counter-argument, we construct two
tableaux, one for $\mathbb{\overline{T}} \ a:C$ and one for
$\mathbb{\overline{F}} \ a:C$. As we can see below, both tableaux
are (weakly)-closed, and therefore form the counter-argument $A_1
= (\{ \overline{\mathbb{C}} \ a:D, \overline{\mathbb{C}} \ a:E \},
\mathbb{C} \ a:C)$. We say that the argument $A_1$ \emph{attacks}
the argument $A_0$ because the former is a counter-argument of the
latter.

\centerline{
\begin{tabular}{ccc}
\Tree
[.{$\mathbb{T} \ a:\neg C$ \\
$\mathbb{T} \ a:C \sqcup D$ \\
$\mathbb{T} \ a:\neg D \sqcup E$ \\
$\mathbb{T} \ a:\neg E$ \\
\fbox{$\mathbb{\overline{T}} \ a:C$}}
   [.{$\mathbb{T} \ a:\neg D$}
      [.{$\mathbb{T} \ a:C$}
            [.{$\times$}
         ]
      ]
      [.{$\mathbb{T} \ a:D$}
         [.{$\mathbb{F} \ a:D$}
            [.{$\otimes_{[a:D]}$}
            ]
         ]
      ]
   ]
   [.{$\mathbb{T} \ a:E$}
      [.{$\mathbb{F} \ a:E$}
         [.{$\otimes_{[a:E]}$}
         ]
      ]
   ]
]
 &&
\Tree
[.{$\mathbb{T} \ a:\neg C$ \\
$\mathbb{T} \ a:C \sqcup D$ \\
$\mathbb{T} \ a:\neg D \sqcup E$ \\
$\mathbb{T} \ a:\neg E$ \\
\fbox{$\mathbb{\overline{F}} \ a:C$}}
   [.{$\mathbb{F} \ a:C$}
      [.{$\times$}
      ]
   ]
]
\end{tabular}
}

The two tableaux forming the counter-argument $A_1$ are closed
under the assumptions: $\overline{\mathbb{C}} \ a:D$ and
${\overline{\mathbb{C}} \ a:E}$. So, $A_1$ is a valid argument if
there exists no valid counter-argument for $\mathbb{C} \ a:D$, and
no counter-argument for $\mathbb{C} \ a:E$.

Argument $A_1$ is equivalent to two other arguments, namely: $A_2
= (\{ \overline{\mathbb{C}} \ a:C, \overline{\mathbb{C}} \ a:E \},
\mathbb{C} \ a:D )$ and $A_3 = (\{ \overline{\mathbb{C}} \ a:C,
\overline{\mathbb{C}} \ a:D \}, \mathbb{C} \ a:E)$. A proof of the
equivalence will be given in the next section, Proposition
\ref{prop:counter-arg}.

The arguments $A_2$ and $A_3$ implied by $A_1$ are both
counter-arguments of $A_1$. Moreover, $A_1$ is a counter-argument
of $A_2$ and $A_3$, and $A_2$ and $A_3$ are counter-arguments of
each other. No other counter-arguments can be identified in this
example. Figure \ref{fig:attack} show all the arguments and the
attack relation, denoted by the arrows, between the arguments.
\end{example}

\begin{figure}[htb]
\centerline{
\includegraphics[scale=1.0]{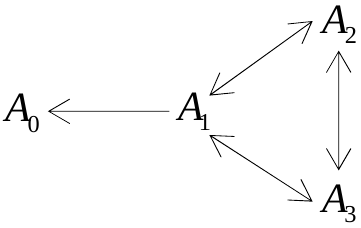}
} \caption{The attack relations between the arguments of Example
\ref{ex:tableaux3}.} \label{fig:attack}
\end{figure}

We will now formally define the arguments and the attack relations
that we can derive from the constructed semantic tableaux.
\begin{definition}
Let $\Sigma$ be set of propositions and let $\overline{\mathbb{C}}
\alpha =$``$\overline{\mathbb{T}} \alpha$ \textit{or}
$\overline{\mathbb{F}} \alpha$'' be an assumption in the argument
$A$. Moreover, let $\mathscr{T}_1$ be a (weakly) closed semantic
tableaux with root $\Gamma_1 = \{ \mathbb{T} \sigma \mid \sigma
\in \Sigma \} \cup \mathbb{\overline{T}} \alpha$ and let
$\mathscr{T}_2$ be a (weakly) closed semantic tableaux with root
$\Gamma_2 = \{ \mathbb{T} \sigma \mid \sigma \in \Sigma \} \cup
\mathbb{\overline{F}} \alpha$. Finally, let $\{
\overline{\mathbb{C}} \alpha_1, \ldots, \overline{\mathbb{C}}
\alpha_k \}$ be the set of assumptions on which the weakly closed
branches in the tableaux $\mathscr{T}_1$ or the tableaux
$\mathscr{T}_2$ are based.

Then $A'=(\{ \overline{\mathbb{C}} \alpha_1, \ldots,
\overline{\mathbb{C}} \alpha_k \}, \mathbb{C} \alpha)$ is a
counter-argument of the argument $A$. We say that the argument
$A'$ \emph{attacks} the argument $A$, denoted by: $A'
\longrightarrow A$.
\end{definition}

The form of argumentation that we have here is called
assumption-based argumentation (ABA), which has been developed
since the end of the 1980's
\cite{BDKT-97,BTK-93,DKT-09,GT-07,Roo-88a,Roo-92}.

Example \ref{ex:tableaux3} shows that an argument can be
counter-argument of an argument and vice versa; e.g., arguments
$A_2$ and $A_3$. This raises the question which arguments are
valid. Argumentation theory and especially the argumentation
framework (AF) introduced by Dung \shortcite{Dun-95} provides an
answer.

Arguments are viewed in an argumentation framework as atoms over
which an attack relation is defined. Figure \ref{fig:attack} shows
the arguments and the attack relations between the arguments
forming the argumentation framework of Example \ref{ex:tableaux3}.
The formal specification of an argumentation framework is given by
the next definition.

\begin{definition}
An argumentation framework is a couple $AF= (
\mathscr{A},\longrightarrow )$ where $\mathscr{A}$ is a finite set
of arguments and $\longrightarrow \subseteq \mathscr{A} \times
\mathscr{A}$ is an attack relation over the arguments.
\end{definition}
For convenience, we extend the attack relation $\longrightarrow$
to sets of arguments.
\begin{definition}
Let $A\in \mathscr{A}$ be an argument and let $\mathscr{S},
\mathscr{P} \subseteq \mathscr{A}$ be two sets of arguments. We
define:
\begin{itemize}
\item
$\mathscr{S} \longrightarrow A$ iff for some $B \in \mathscr{S}$,
$B \longrightarrow A$.
\item
$A \longrightarrow \mathscr{S}$ iff for some $B \in \mathscr{S}$,
$A \longrightarrow B$.
\item
$\mathscr{S} \longrightarrow \mathscr{P}$ iff for some $B \in
\mathscr{S}$ and $C \in \mathscr{P}$, $B \longrightarrow C$.
\end{itemize}
\end{definition}

Dung \shortcite{Dun-95} describes different argumentation
semantics for an argumentation framework in terms of sets of
acceptable arguments. These semantics are based on the idea of
selecting a coherent subset $\mathscr{E}$ of the set of arguments
$\mathscr{A}$ of the argumentation framework $AF= (
\mathscr{A},\longrightarrow )$. Such a set of arguments
$\mathscr{E}$ is called an \emph{argument extension}. The
arguments of an argument extension support propositions that give
a coherent description of what might hold in the world. Clearly, a
basic requirement of an argument extension is being
\emph{conflict-free}; i.e., no argument in an argument extension
attacks another argument in the argument extension. Besides being
conflict-free, an argument extension should defend itself against
attacking arguments by attacking the attacker.

\begin{definition}
Let $AF= ( \mathscr{A},\longrightarrow )$ be an argumentation
framework and let $\mathscr{S} \subseteq \mathscr{A}$ be a set of
arguments.
\begin{itemize}
\item
$\mathscr{S}$ is \emph{conflict-free} iff $\mathscr{S}
\centernot\longrightarrow \mathscr{S}$.
\item
$\mathscr{S}$ defends an argument $A \in \mathscr{A}$ iff for
every argument $B \in \mathscr{A}$ such that $B \longrightarrow
A$, $\mathscr{S} \longrightarrow B$.
\end{itemize}
\end{definition}

Not every conflict-free set of arguments that defends itself, is
considered to be an argument extension. Several additional
requirements have been formulated by Dung \shortcite{Dun-95},
resulting in three different semantics: the \emph{stable}, the
\emph{preferred} and the \emph{grounded} semantics.

\begin{definition}
Let $AF= ( \mathscr{A},\longrightarrow )$ be an argumentation
framework and let $\mathscr{E} \subseteq \mathscr{A}$.
\begin{itemize}
\item
$\mathscr{E}$ is a \emph{stable extension} iff $\mathscr{E}$ is
conflict-free,
and for every argument $A \in (\mathscr{A} \setminus
\mathscr{E})$, $\mathscr{E} \longrightarrow A$; i.e.,
$\mathscr{E}$ defends itself against every possible attack by
arguments in $\mathscr{A} \backslash \mathscr{E}$.
\item
$\mathscr{E}$ is a \emph{preferred extension} iff $\mathscr{E}$ is
maximal (w.r.t. $\subseteq$) set of arguments that (1) is
conflict-free,
and (2) $\mathscr{E}$ defends every argument $A \in \mathscr{E}$.
\item
$\mathscr{E}$ is a \emph{grounded extension} iff $\mathscr{E}$ is
the minimal (w.r.t. $\subseteq$) set of arguments that (1) is
conflict-free, (2) defends every argument $A \in \mathscr{E}$, and
(3) contains all arguments in $\mathscr{A}$ it defends.
\end{itemize}
\end{definition}

We are interested in stable semantics. We will show in the next
section that stable extensions correspond to conflict-minimal
interpretations. More specifically, we will prove that a
conclusion supported by an argument in every stable extension, is
entailed by every conflict-minimal interpretation, and vice versa.


Is it possible that a conclusion is supported by a different
argument in every stable extension? The answer is Yes, as is
illustrated by Example \ref{ex:tableaux4}. In this example we have
two arguments supporting the conclusion $a:E$, namely $A_0$ and
$A_1$. As can be seen in Figure \ref{fig:attack2}, there are two
stable extensions of the argumentation framework. One extension
contains the argument $A_0$ and the other contains the argument
$A_1$. So, in every extension there is an argument supporting the
conclusion $a:E$. Hence, $\Sigma \models_{\leq_c} a:E$.

\begin{example} \label{ex:tableaux4}
Let $\Sigma = \{ a:\neg C, a:C \sqcup D, a:\neg D, a:C \sqcup E,
a:D \sqcup E \}$ be a set of propositions. The following two
tableaux imply the two arguments $A_0  = (\{ \overline{\mathbb{C}}
\ a:C \}, \mathbb{T} \ a:E)$ and $A_1 = (\{ \overline{\mathbb{C}}
\ a:D \}, \mathbb{T} \ a:E)$, both supporting the conclusion
$a:E$:

\begin{center}
\begin{tabular}{ccc}
\Tree
[.{$\mathbb{T} \ a:\neg C$ \\
$\mathbb{T} \ a:C \sqcup D$ \\
$\mathbb{T} \ a:\neg D$ \\
$\mathbb{T} \ a:C \sqcup E$ \\
$\mathbb{T} \ a:D \sqcup E$ \\
\fbox{$\mathbb{\overline{T}} \ a:E$}}
   [.{$\mathbb{F} \ a:C$}
         [.{$\mathbb{T} \ a:C$}
            [.{$\otimes_{[a:C]}$}
            ]
         ]
         [.{$\mathbb{T} \ a:E$}
            [.{$\times$}
            ]
         ]
   ]
]

& &

\Tree
[.{$\mathbb{T} \ a:\neg C$ \\
$\mathbb{T} \ a:C \sqcup D$ \\
$\mathbb{T} \ a:\neg D$ \\
$\mathbb{T} \ a:C \sqcup E$ \\
$\mathbb{T} \ a:D \sqcup E$ \\
\fbox{$\mathbb{\overline{T}} \ a:E$}}
      [.{$\mathbb{F} \ a:D$}
         [.{$\mathbb{T} \ a:D$}
            [.{$\otimes_{[a:D]}$}
            ]
         ]
         [.{$\mathbb{T} \ a:E$}
            [.{$\times$}
            ]
         ]
      ]
]
\end{tabular}
\end{center}

The assumption $\overline{\mathbb{C}} \ a:C$ in argument $A_0$
makes it possible to determine a counter-argument $A_2 = (\{
\overline{\mathbb{C}} \ a:D \}, \mathbb{C} \ a:C)$ using of the
following two tableaux:

\begin{center}
\begin{tabular}{ccc}
\Tree
[.{$\mathbb{T} \ a:\neg C$ \\
$\mathbb{T} \ a:C \sqcup D$ \\
$\mathbb{T} \ a:\neg D$ \\
$\mathbb{T} \ a:C \sqcup E$ \\
$\mathbb{T} \ a:D \sqcup E$ \\
\fbox{$\mathbb{\overline{T}} \ a:C$}}
      [.{$\mathbb{F} \ a:D$}
         [.{$\mathbb{T} \ a:C$}
            [.{$\times$}
            ]
         ]
         [.{$\mathbb{T} \ a:D$}
            [.{$\otimes_{[a:D]}$}
            ]
         ]
      ]
]

& &

\Tree
[.{$\mathbb{T} \ a:\neg C$ \\
$\mathbb{T} \ a:C \sqcup D$ \\
$\mathbb{T} \ a:\neg D$ \\
$\mathbb{T} \ a:C \sqcup E$ \\
$\mathbb{T} \ a:D \sqcup E$ \\
\fbox{$\mathbb{\overline{F}} \ a:C$}}
   [.{$\mathbb{F} \ a:C$}
      [.{$\times$}
      ]
   ]
]
\end{tabular}
\end{center}

According to Proposition \ref{prop:counter-arg}, $A_2$ implies the
counter-argument $A_3 = (\{ \overline{\mathbb{C}} \ a:C \},
\mathbb{C} \ a:D )$ of $A_1$ and $A_2$. $A_2$ is also a
counter-argument of $A_3$. Figure \ref{fig:attack2} shows the
attack relations between the arguments $A_0$, $A_1$, $A_2$ and
$A_3$.
\end{example}

\begin{figure}[htb]
\centerline{
\includegraphics[scale=1.0]{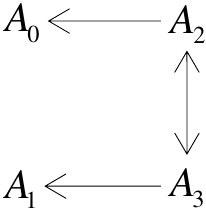}
} \caption{The attack relations between the arguments of Example
\ref{ex:tableaux4}.} \label{fig:attack2}
\end{figure}

Example \ref{ex:interpretation} gives an illustration of the
semantic interpretations of Example \ref{ex:tableaux4}. The
example shows two conflict-minimal interpretations. These
conflict-minimal interpretations correspond with the two
\emph{stable extensions}. Interpretation $I_1$ entails $a:E$
because $I_1$ must entail $a:C \sqcup E$ and $I_1$ does not entail
$a:C$, and interpretation $I_2$ entails $a:E$ because $I_2$ must
entail $a:D \sqcup E$ and $I_2$ does not entail $a:D$.


\begin{example} \label{ex:interpretation}
Let $\Sigma = \{ a:\neg C, a:C \sqcup D, a:\neg D, a:C \sqcup E,
a:D \sqcup E \}$ be a set of propositions. There are two
conflict-minimal interpretations containing the following
interpretation functions:
\begin{itemize}
\item
$\pi_1(a:C)=\{f\}$, $\pi_1(a:D)=\{t,f\}$, $\pi_1(a:E)=\{t\}$.
\item
$\pi_2(a:C)=\{t,f\}$, $\pi_2(a:D)=\{f\}$, $\pi_2(a:E)=\{t\}$.
\end{itemize}
In both interpretations $a:E$ is entailed.
\end{example}


\section{Correctness and completeness proofs}
In this section we investigate whether the proposed approach is
correct. That is whether a proposition supported by an argument in
every stable extension is entailed by every conflict-minimal
interpretation. Moreover, we investigate whether the approach is
complete. That is, whether a proposition entailed by every
conflict-minimal interpretation is supported by an argument in
every stable extension.

In the following theorem we will use the notion of ``a complete
set of arguments relevant to $\varphi$''. This set of arguments
$\mathscr{A}$ consists of all argument $A$ supporting $\varphi$,
all possible counter-arguments, all possible counter arguments of
the counter-arguments, etc.
\begin{definition}
A complete set of arguments $\mathscr{A}$ relevant to $\varphi$
satisfies the following requirements:
\begin{itemize}
\item
$\{ A \mid A \mbox{ supports } \varphi \} \subseteq \mathscr{A}$.
\item
If $A \in \mathscr{A}$ and $B$ is a counter-argument of $A$ that
we can derive, then $B \in \mathscr{A}$ and $(B, A) \in \
\longrightarrow$.
\end{itemize}
\end{definition}

\begin{theorem}[correctness and completeness] \label{th:correctness-AF}
Let $\Sigma$ be a set of propositions and let $\varphi$ be a
proposition. Moreover, let $\mathscr{A}$ be a complete set of
arguments relevant to $\varphi$, let $\longrightarrow \subseteq
\mathscr{A} \times \mathscr{A}$ be the attack relation determined
by $\mathscr{A}$, and let \mbox{$(\mathscr{A},\longrightarrow)$}
be the argumentation framework. Finally, let
$\mathscr{E}_1,\ldots,\mathscr{E}_k$ be all stable extensions of
the argumentation framework $(\mathscr{A}, \longrightarrow)$.

$\Sigma$ entails the proposition $\varphi$ using the
conflict-minimal three-valued semantics; i.e., $\Sigma
\models_{\leq_c} \varphi$, iff $\varphi$ is supported by an
argument in every stable extension $\mathscr{E}_i$ of
$(\mathscr{A}, \longrightarrow)$.
\end{theorem}

To prove Theorem \ref{th:correctness-AF}, we need the following
lemmas. In these lemmas we will use the following notations: We
will use $I \models \mathbb{T} \alpha$ to denote that $t \in
I(\alpha)$ ( $I \models \alpha$ ), and $I \models \mathbb{F}
\alpha$ to denote that $f \in I(\alpha)$. Moreover, we will use
$\Sigma \models \mathbb{T} \alpha$ and $\Sigma \models \mathbb{F}
\alpha$ to denote that $\mathbb{T} \alpha$ and $\mathbb{F}
\alpha$, respectively, hold in all three-valued interpretations of
$\Sigma$.

The first lemma proves the correctness of the arguments in
$\mathscr{A}$.


\begin{lemma}[correctness of arguments] \label{lem:argument-soundness}
Let $\Sigma$ be a set of propositions and let $\varphi$ be a
proposition. Moreover, let $\mathbb{L}$ be either the label
$\mathbb{T}$ or $\mathbb{F}$.

If a semantic tableaux with root $\Gamma = \{ \mathbb{T} \sigma
\mid \sigma \in \Sigma \} \cup \{ \overline{\mathbb{L}} \varphi
\}$ is weakly closed, and if $\{ \overline{\mathbb{C}} \alpha_1,
\ldots, \overline{\mathbb{C}} \alpha_k \}$ is the set of weak
closure assumptions implied by all the weakly closed leafs of the
tableaux, then
\[ \{ \overline{\mathbb{C}} \alpha_1, \ldots, \overline{\mathbb{C}}
\alpha_k \} \cup \{ \mathbb{T} \sigma \mid \sigma \in \Sigma \}
\models\mathbb{L} \varphi \]
\end{lemma}

\begin{proof}
\begin{itemize}
\item[]
Suppose that $\{ \overline{\mathbb{C}} \alpha_1, \ldots,
\overline{\mathbb{C}} \alpha_k \} \cup \{ \mathbb{T} \sigma \mid
\sigma \in \Sigma \} \not\models\mathbb{L} \varphi$. Then there
must be an interpretation $I$ satisfying $\{ \overline{\mathbb{C}}
\alpha_1, \ldots, \overline{\mathbb{C}} \alpha_k \} \cup \{
\mathbb{T} \sigma \mid \sigma \in \Sigma\}$ but not $\mathbb{L}
\varphi$. So, $I \models\{ \overline{\mathbb{C}} \alpha_1, \ldots,
\overline{\mathbb{C}} \alpha_k \} \cup \{ \mathbb{T} \sigma \mid
\sigma \in \Sigma \} \cup \{ \overline{\mathbb{L}} \varphi \}$. We
can create a tableaux for $\{ \overline{\mathbb{C}} \alpha_1,
\ldots, \overline{\mathbb{C}} \alpha_k \} \cup \{ \mathbb{T}
\sigma \mid \sigma \in \Sigma \} \cup \{ \overline{\mathbb{L}}
\varphi \}$ by adding the assumptions $\{ \overline{\mathbb{C}}
\alpha_1, \ldots, \overline{\mathbb{C}} \alpha_k \}$ to every node
in the original tableaux with root $\Gamma$. Let $\Gamma^* = \{
\overline{\mathbb{C}} \alpha_1, \ldots, \overline{\mathbb{C}}
\alpha_k \} \cup \{ \mathbb{T} \sigma \mid \sigma \in \Sigma \}
\cup \{ \overline{\mathbb{L}} \varphi \}$ be the root of the
resulting tableaux. Since $I \models\Gamma^*$, there must be a
leaf $\Lambda^*$ of the new tableaux and $I \models\Lambda^*$. The
corresponding leaf $\Lambda$ in the original tableaux with root
$\Gamma$ is either strongly or weakly closed.
\begin{itemize}
\item
If $\Lambda$ is strongly closed, then so is $\Lambda^*$ and we
have a contradiction.
\item
If $\Lambda$ is weakly closed, then the weak closure  implies one
of the assumptions $\overline{\mathbb{C}} \alpha_i$ because
$\{\mathbb{T} \alpha_i, \mathbb{F} \alpha_i \} \subseteq \Lambda$.
Therefore, $\{\mathbb{T} \alpha_i, \mathbb{F} \alpha_i \}
\subseteq \Lambda^*$. Since $\{\mathbb{T} \alpha_i, \mathbb{F}
\alpha_i \}$ implies $\mathbb{C} \alpha_i$ and since
$\overline{\mathbb{C}} \alpha_i \in \Lambda^*$, $I
\not\models\Lambda^*$ The latter contradicts with $I
\models\Lambda^*$.
\end{itemize}
\end{itemize}
Hence, the lemma holds.
\end{proof}
The above lemma implies that the assumptions of an argument $A =
(\{ \overline{\mathbb{C}} \alpha_1, \ldots, \overline{\mathbb{C}}
\alpha_k \}, \mathbb{L} \varphi)$ together with $\Sigma$ entail
the conclusion of $A$.


The next lemma proves the completeness of the set of arguments
$\mathscr{A}$.

\begin{lemma}[completeness of arguments] \label{lem:argument-completeness}
Let $\Sigma$ be a set of propositions and let $\varphi$ be a
proposition. Moreover, let $\mathbb{L}$ be either the label
$\mathbb{T}$ or $\mathbb{F}$.

If $\{ \overline{\mathbb{C}} \alpha_1, \ldots,
\overline{\mathbb{C}} \alpha_k \}$ is a set of  atomic assumptions
with $\alpha_i = a_i:C_i$, $a_i \in \mathbf{N}$ and $C_i \in
\mathbf{C_i}$, and if
\[ \{ \overline{\mathbb{C}} \alpha_1, \ldots, \overline{\mathbb{C}}
\alpha_k \} \cup \{ \mathbb{T} \sigma \mid \sigma \in \Sigma \}
\models\mathbb{L} \varphi \]
then there is a semantic tableaux with root $\Gamma = \{
\mathbb{T} \sigma \mid \sigma \in \Sigma \} \cup \{
\overline{\mathbb{L}} \varphi \}$, and the tableaux is weakly
closed.
\end{lemma}

\begin{proof}
Let $\Gamma = \{ \mathbb{T} \sigma \mid \sigma \in \Sigma \} \cup
\{ \overline{\mathbb{L}} \varphi \}$ be the root of a semantic
tableaux.
\begin{itemize}
\item[]
Suppose that the tableaux is \emph{not} weakly closed. Then there
is an open leaf $\Lambda$. We can create a tableaux for $\{
\overline{\mathbb{C}} \alpha_1, \ldots, \overline{\mathbb{C}}
\alpha_k \} \cup \{ \mathbb{T} \sigma \mid \sigma \in \Sigma \}
\cup \{ \overline{\mathbb{L}} \varphi \}$ by adding the
assumptions $\{ \overline{\mathbb{C}} \alpha_1, \ldots,
\overline{\mathbb{C}} \alpha_k \}$ to every node in the original
tableaux with root $\Gamma$. Let $\Gamma^* = \{
\overline{\mathbb{C}} \alpha_1, \ldots, \overline{\mathbb{C}}
\alpha_k \} \cup \{ \mathbb{T} \sigma \mid \sigma \in \Sigma \}
\cup \{ \overline{\mathbb{L}} \varphi \}$ be the root of the
resulting tableaux. Since $\{ \overline{\mathbb{C}} \alpha_1,
\ldots, \overline{\mathbb{C}} \alpha_k \} \cup \{ \mathbb{T}
\sigma \mid \sigma \in \Sigma \} \models\mathbb{L} \varphi$, there
exists no interpretation $I$ such that $I \models\Gamma^*$.
Therefore, there exists no interpretation $I$ such that $I
\models\Lambda^*$. Since we considered only atomic assumptions
$\overline{\mathbb{C}} \alpha_i$, we cannot extend the tableaux by
rewriting a proposition in $\Lambda^*$. Therefore, $\Lambda^*$
must be strongly closed and for some $\alpha_i$, $\{ \mathbb{T}
\alpha_i, \mathbb{F} \alpha_i \} \subseteq \Lambda^*$. This
implies that $\{ \mathbb{T} \alpha_i, \mathbb{F} \alpha_i \}
\subseteq \Lambda$. Hence, $\Lambda$ is weakly closed under the
assumption $\overline{\mathbb{C}} \alpha_i$. Contradiction.
\end{itemize}
Hence, the lemma holds.
\end{proof}
The above lemma implies that we can find an argument $A = (\{
\overline{\mathbb{C}} \alpha_1, \ldots, \overline{\mathbb{C}}
\alpha_k \}, \mathbb{L} \varphi)$ for any set of assumption that,
together with $\Sigma$, entails a conclusion $\mathbb{L} \varphi$.


The following lemma proves that for every conflict $\mathbb{C}
\varphi$ entailed by a conflict-minimal interpretation, we can
find an argument supporting $\mathbb{C} \varphi$ of which the
assumptions are entailed by the conflict-minimal interpretation.
\begin{lemma} \label{lem:cf-int-argument}
Let $\Sigma$ be a set of propositions and let $I=\langle O,\pi
\rangle$ be a conflict-minimal interpretation of $\Sigma$.
Moreover, let $\varphi$ be a proposition.

If $I \models \mathbb{C} \varphi$ holds, then there is an argument
$A = (\{ \overline{\mathbb{C}} \alpha_1, \ldots,
\overline{\mathbb{C}} \alpha_k \}, \mathbb{C} \varphi)$ supporting
$\mathbb{C} \varphi$ and for every assumption
$\overline{\mathbb{C}} \alpha_i$, $I \models \overline{\mathbb{C}}
\alpha_i$ holds.
\end{lemma}

\begin{proof}
Let $I$ be a conflict-minimal interpretation of $\Sigma$.
\begin{itemize}
\item[]
Suppose that $I \models \mathbb{C} \varphi$ holds. We can
construct a tableaux for:
\begin{eqnarray*}
\Gamma &=& \{ \mathbb{T} \sigma \mid \sigma \in \Sigma \}
 \cup \{ \overline{\mathbb{C}} \varphi \} \cup \\
&& \{ \overline{\mathbb{C}} \ a:C \mid C \in \mathbf{C}, \pi(a:C)
\not= \{t,f\} \}
\end{eqnarray*}
\begin{itemize}
\item[]
Suppose that this tableaux is not strongly closed. Then there is
an interpretation $I'=\langle O,\pi' \rangle$ satisfying the root
$\Gamma$. Clearly, $I' <_c I$ because for every $a:C$ with $C \in
\mathbf{C}$, if $\pi(a:C) \not= \{t,f\}$, then $\pi'(a:C) \not=
\{t,f\}$. Since $I$ is a conflict-minimal interpretation and since
$I' \not\models \mathbb{C} \varphi$, we have a contradiction.
\end{itemize}
Hence, the tableaux is closed.

Since the tableaux with root $\Gamma$ is closed, we can identify
all assertions in $\{ \overline{\mathbb{C}} \ a:C \mid C \in
\mathbf{C}, \pi(a:C) \not= \{t,f\} \}$ that are \textbf{not} used
to close a leaf of the tableaux. These assertions
$\overline{\mathbb{C}} \ a:C$ play no role in the construction of
the tableaux and can therefore be removed from every node of the
tableaux. The result is still a valid and closed semantic tableaux
with a new root $\Gamma'$. The assertions in $\{
\overline{\mathbb{C}} \ a:C \mid C \in \mathbf{C}, \pi(a:C) \not=
\{t,f\} \} \cap \Gamma'$ must all be used to strongly close leafs
of the tableaux $\Gamma'$, and also of $\Gamma$. A leaf that is
strongly closed because of $\overline{\mathbb{C}} \ a:C$ can be
closed weakly under the assumption $\overline{\mathbb{C}} \ a:C$.
So, we may remove the remaining assertions $\overline{\mathbb{C}}
\ a:C$ from the root $\Gamma'$. The result is still a valid
semantic tableaux with root $\Gamma'' = \{ \mathbb{T} \sigma \mid
\sigma \in \Sigma \} \cup \{ \overline{\mathbb{C}} \varphi \}$.
This tableaux with root $\Gamma''$ is weakly closed, and by the
construction of the tableaux, $I \models\overline{\mathbb{C}} \
a:C$ holds for every assumption $\overline{\mathbb{C}} \ a:C$
implied by a weak closure. Hence, we have constructed an argument
$A = (\{ \overline{\mathbb{C}} \alpha_1, \ldots,
\overline{\mathbb{C}} \alpha_k \}, \mathbb{C} \varphi)$ supporting
$\mathbb{C} \varphi$ and for every assumption
$\overline{\mathbb{C}} \alpha_i$, $I \models \overline{\mathbb{C}}
\alpha_i$ holds.
\end{itemize}
Hence, the lemma holds.
\end{proof}


For the next lemma we need the following definition of a set of
assumptions that is allowed by an extension.
\begin{definition}
Let $\Omega$ be the set of all assumptions
$\overline{\mathbb{C}}\alpha$ in the arguments $\mathscr{A}$. For
any extension $\mathscr{E} \subseteq \mathscr{A}$,
\[ \Omega(\mathscr{E}) = \{\overline{\mathbb{C}}\alpha \in \Omega
\mid \mbox{no argument }A \in \mathscr{E} \mbox{ supports }
\mathbb{C}\alpha \} \] is the set of assumptions allowed by the
extension $\mathscr{E}$.
\end{definition}


The last lemma proves that for every conflict-minimal
interpretation there is a corresponding stable extension.

\begin{lemma} \label{lem:stable-semantics}
Let $\Sigma$ be a set of propositions and let $\varphi$ be a
proposition. Moreover, let $\mathscr{A}$ be the complete set of
arguments relevant to $\varphi$, let $\longrightarrow \subseteq
\mathscr{A} \times \mathscr{A}$ be the attack relation determined
by $\mathscr{A}$, and let $(\mathscr{A},\longrightarrow)$ be the
argumentation framework.

For every conflict-minimal interpretation $I$ of $\Sigma$, there
is a stable extension $\mathscr{E}$ of
$(\mathscr{A},\longrightarrow)$ such that $I
\models\Omega(\mathscr{E})$.
\end{lemma}

\begin{proof}
Let $I$ be a conflict-minimal interpretation 
and let %
\[ \mathscr{E} =\{ (\{ \overline{\mathbb{C}} \alpha_1, \ldots,
\overline{\mathbb{C}} \alpha_k \}, \varphi) \in \mathscr{A} \mid I
\models \{ \overline{\mathbb{C}} \alpha_1, \ldots,
\overline{\mathbb{C}} \alpha_k \} \} \]%
be the set of arguments $A=(\{ \overline{\mathbb{C}} \alpha_1,
\ldots, \overline{\mathbb{C}} \alpha_k \}, \varphi)$ of which the
assumptions are entailed by $I$.
\begin{itemize}
\item[]
Suppose $\mathscr{E}$ is not conflict-free. Then there is an
argument $B \in \mathscr{E}$ such that $B \longrightarrow A$ with
$A \in \mathscr{E}$. So, $B$ supports $\mathbb{C} \psi$ and
$\overline{\mathbb{C}} \psi$ is an assumption of $A$. Since $I$
entails the assumptions of $A$, $I \not\models \mathbb{C} \psi$.
Since $I$ is a conflict-minimal interpretation of $\Sigma$
entailing the assumptions of $B$, according to Lemma
\ref{lem:argument-soundness}, $I \models \mathbb{C} \psi$.
Contradiction.
\end{itemize}
Hence, $\mathscr{E}$ is a conflict-free set of argument.

\begin{itemize}
\item[]
Suppose that there exists an argument  $A \in \mathscr{A}$ such
that $A \not\in \mathscr{E}$. Then, for some assumption
$\overline{\mathbb{C}} \alpha$ of $A$, $I \not\models
\overline{\mathbb{C}} \alpha$. So, $I \models \mathbb{C} \alpha$,
and according to Lemma \ref{lem:cf-int-argument}, there is an
argument $B \in \mathscr{E}$ supporting $\mathbb{C} \alpha$.
Therefore, $B \longrightarrow A$.
\end{itemize}
Hence, $\mathscr{E}$ attacks every argument $A \in
\mathscr{A}\backslash\mathscr{E}$. Since $\mathscr{E}$ is also
conflict-free, $\mathscr{E}$ is a \emph{stable} extension of
$(\mathscr{A},\longrightarrow)$.

\begin{itemize}
\item[]
Suppose that $I \not\models\Omega(\mathscr{E})$. Then there is a
$\overline{\mathbb{C}} \alpha \in \Omega(\mathscr{E})$ and $I
\models \mathbb{C} \alpha$. According to Lemma
\ref{lem:cf-int-argument}, there is an argument $A=(\{
\overline{\mathbb{C}} \alpha_1, \ldots, \overline{\mathbb{C}}
\alpha_k \}, \mathbb{C} \alpha)$ and $I \models \{
\overline{\mathbb{C}} \alpha_1, \ldots, \overline{\mathbb{C}}
\alpha_k \}$. So, $A \in \mathscr{E}$ and therefore,
$\overline{\mathbb{C}} \alpha \not\in \Omega(\mathscr{E})$.
Contradiction.
\end{itemize}
Hence, $I \models \Omega(\mathscr{E})$.
\end{proof}


 Using the results of the above lemmas, we can now prove the theorem.

\begin{proof}[of Theorem \ref{th:correctness-AF}] \ \\
($\Rightarrow$) Let $\Sigma \models_{\leq_c} \varphi$.
\begin{itemize}
\item[]
Suppose that there is stable extension $\mathscr{E}_i$ that does
not contain an argument for $\varphi$. Then according to Lemma
\ref{lem:argument-completeness}, $\{ \mathbb{T} \sigma \mid \sigma
\in \Sigma \} \cup \Omega(\mathscr{E}_i) \not\models\mathbb{T}
\varphi$. So, there exists an interpretation $I$ such that $I
\models\{ \mathbb{T} \sigma \mid \sigma \in \Sigma \} \cup
\Omega(\mathscr{E}_i)$ but $I \not\models\mathbb{T} \varphi$.
There must also exists a conflict-minimal interpretation $I'$ of
$\Sigma$ and $I' \leq_c I$.  Since the assumptions
$\overline{\mathbb{C}} \ a:C \in \Omega(\mathscr{E}_i)$ all state
that there is no conflict concerning the assertion $a:C$, $I'
\models\Omega(\mathscr{E}_i)$ must hold. So, $I'$ is a
conflict-minimal interpretation of $\Sigma$ and $I'
\models\Omega(\mathscr{E}_i)$ but according to Lemma
\ref{lem:argument-completeness}, $I' \not\models\mathbb{T}
\varphi$. This implies $\Sigma \not\models_{\leq_c} \varphi$.
Contradiction.
\end{itemize}
Hence, every stable extension $\mathscr{E}_i$ contains an argument
for $\varphi$.

($\Leftarrow$) Let $\varphi$ be supported by an argument in every
stable extension $\mathscr{E}_i$.
\begin{itemize}
\item[]
Suppose that $\Sigma \not\models_{\leq_c} \varphi$. Then there is
a conflict-minimal interpretation $I$ of $\Sigma$ and $I
\not\models \varphi$. Since $I$ is a conflict-minimal
interpretation of $\Sigma$, according to Lemma
\ref{lem:stable-semantics}, there is a stable extension
$\mathscr{E}_i$ and $I \models\Omega(\mathscr{E}_i)$. Since
$\mathscr{E}_i$ contains an argument $A$ supporting $\varphi$, the
assumptions of $A$ must be a subset of $\Omega(\mathscr{E}_i)$,
and therefore $I$ satisfies these assumptions. Then, according to
Lemma \ref{lem:argument-soundness}, $I \models\varphi$.
Contradiction.
\end{itemize}
Hence, $\Sigma \models_{\leq_c} \varphi$.
\end{proof}


In Example \ref{ex:tableaux3} in the previous section, we saw that
one counter-argument implies multiple counter-arguments. The
following proposition formalizes this observation.
\begin{proposition} \label{prop:counter-arg}
Let $A_0=(\{ \overline{\mathbb{C}} \alpha_1, \ldots,
\overline{\mathbb{C}} \alpha_k \}, \mathbb{C} \alpha_0)$.

Then $A_i=(\{ \overline{\mathbb{C}} \alpha_0, \ldots,
\overline{\mathbb{C}} \alpha_{i-1}, \overline{\mathbb{C}}
\alpha_{i+1}, \ldots, \overline{\mathbb{C}} \alpha_k \},
\mathbb{C} \alpha_i)$ is  an argument for every $1\leq i \leq k$.
\end{proposition}

\begin{proof}
The argument $A_0$ is the result of two tableaux, one for
$\mathbb{T} \alpha_0$ and one for $\mathbb{F} \alpha_0$. Then,
according to Lemma \ref{lem:argument-soundness},
\[ \{ \overline{\mathbb{C}} \alpha_1, \ldots, \overline{\mathbb{C}}
\alpha_k \} \cup \{ \mathbb{T} \sigma \mid \sigma \in \Sigma \}
\models\mathbb{C} \alpha_0 \]%
where $\Sigma$ the set of available propositions. This implies
that
\[ \{ \overline{\mathbb{C}} \alpha_0, \ldots, \overline{\mathbb{C}}
\alpha_{i-1}, \overline{\mathbb{C}} \alpha_{i+1},
\overline{\mathbb{C}} \alpha_k \} \cup \{ \mathbb{T} \sigma \mid
\sigma \in \Sigma \} \models\mathbb{C} \alpha_i\]%
So, $\{ \overline{\mathbb{C}} \alpha_0, \ldots,
\overline{\mathbb{C}} \alpha_{i-1}, \overline{\mathbb{C}}
\alpha_{i+1}, \overline{\mathbb{C}} \alpha_k \} \cup \{ \mathbb{T}
\sigma \mid \sigma \in \Sigma \}$ entails both $\mathbb{T}
\alpha_i$ and $\mathbb{F} \alpha_i$. Then, according to Lemma
\ref{lem:argument-completeness},
\[ A_i=(\{ \overline{\mathbb{C}} \alpha_0, \ldots,
\overline{\mathbb{C}} \alpha_{i-1}, \overline{\mathbb{C}}
\alpha_{i+1}, \ldots, \overline{\mathbb{C}} \alpha_k \},
\mathbb{C} \alpha_i)\]%
is an argument for $\mathbb{C} \alpha_i$.
\end{proof}

\section{Related Works} \label{sec:related-works}
Reasoning in the presences of inconsistent information has been
addressed using different approaches. Rescher \shortcite{Res-64}
proposed to focus on maximal consistent subsets of an inconsistent
knowledge-base. This proposal was further developed by
\cite{Bre-89,HHT-05,Poo-88,Roo-88a,Roo-92}. Brewka and Roos focus
on preferred maximal consistent subsets of the knowledge-base
while Poole and Huang et al.~consider a single consistent subset
of the knowledge-base supporting a conclusion. Roos
\shortcite{Roo-92} defines a preferential semantics
\cite{KRa-90,Mak-94,Sho-87} entailing the conclusions that are
entailed by every preferred maximal consistent subsets, and
provides an assumption-based argumentation system capable of
identifying the entailed conclusions.

Paraconsistent logics form another approach to handle inconsistent
knowledge bases. Paraconsistent logics have a long history
starting with Aristotle. From the beginning of the twentieth
century, paraconsistent logics were developed by Orlov (1929),
Asenjo \shortcite{Ase-66}, da Costa \shortcite{Cos-74}, Belnap
\shortcite{Bel-77}, Priest \shortcite{Pri-89} and others. For a
survey of several paraconsistent logics, see for instance
\cite{Mid-11}.

This paper uses the semantics of the paraconsistent logic LP
\cite{Pri-89,Pri-91} as starting point. Belnap's four-values
semantics \shortcite{Bel-77} differs from the LP semantics in
allowing the empty set of truth-values. Belnap's semantics was
adapted to description logics by Patel-Schneider
\shortcite{Pat-89}. Ma et al.
\shortcite{MLL-06,MHL-07,MHL-08,MH-09} extend Patel-Schneider's
work to more expressive description logics, and propose two new
interpretations for the subsumption relation. Qiao and Roos
\shortcite{QR-11} propose another interpretation.

A proof theory based on the semantic tableaux method was first
introduced by Beth \shortcite{Bet-55}. The semantic tableaux
methods have subsequently been developed for many logics. For an
overview of several semantic tableaux methods, see \cite{Hah-01}.
Bloesch \shortcite{Blo-93} developed a semantic tableaux method
for the paraconsistent logics LP and Belnap's 4-valued logic. This
semantic tableaux method has been used as a starting point in this
paper.

Argumentation theory has its roots in logic and rhetoric. It dates
back to Greek philosophers such as Aristotle. Modern argumentation
theory started with the work of Toulmin \shortcite{Tou-58}. In
Artificial Intelligence, the use of argumentation was promoted by
authors such as Pollock \shortcite{Pol-87}, Simari and Loui
\shortcite{SL-92}, and others. Dung \shortcite{Dun-95} introduced
the argumentation framework (AF) in which he abstracts from the
structure of the argument and the way the argument is derived. In
Dung's argumentation framework, arguments are represented by atoms
over which an attack relation is defined. The argumentation
framework is used to define an argumentation semantics in terms of
sets of conflict-free arguments that defend themselves against
attacking arguments. Dung defines three semantics for
argumentation frameworks: the grounded, the stable and the
preferred semantics. Other authors have proposed additional
semantics to overcome some limitations of these three semantics.
For an overview, see \cite{BCD-07}.

This paper uses a special type argumentation system called
assumption-based argumentation (ABA). Assumption-based
argumentation has been developed since the end of the 1980's
\cite{BDKT-97,BTK-93,GT-07,Roo-88a,
Roo-92}. Dung et al. \shortcite{DKT-09} formalized
assumption-based argumentation in terms of an argumentation
framework.

\section{Conclusions}\label{sec:conclusions}
This paper presented a three-valued semantics for $\mathcal{ALC}$,
which is based on semantics of the paraconsistent logic LP. An
assumption-based argumentation system for identifying conclusions
supported by conflict-minimal interpretations was subsequently
described. The assumption-based arguments are derived from open
branches of a semantic tableaux. The assumptions close open
branches by assuming that some proposition will not be assigned
the truth-value \textsc{conflict}. No assumptions are needed if
conclusions hold is all three-valued interpretations. The
described approach has also been implemented.

In future work we intend to extend the approach to the description
logic $\mathcal{SROIQ}$. Moreover, we wish to investigate the
computational efficiency of our approach in handling
inconsistencies.

\bibliographystyle{aaai}
\small
\bibliography{NMR,PL,DL}


\end{document}